\def\year{2021}\relax
\documentclass[letterpaper]{article} 
\usepackage{aaai21}  
\usepackage{times}  
\usepackage{helvet} 
\usepackage{courier}  
\usepackage{subcaption}
\usepackage[hyphens]{url}  
\usepackage{graphicx} 
\usepackage{natbib}  
\usepackage{caption} 
\usepackage{uzlcolor} 
\usepackage{relsize} 
\usepackage{tikz} 
\usepackage{xstring} 
\usepackage{float} 
\usepackage{tiles} 
\usetikzlibrary{ 
  arrows,
  mindmap,
  positioning,
  overlay-beamer-styles,
  arrows.meta,
  automata,
  backgrounds,
  calc,
  chains,
  spy,
  decorations,
  decorations.pathmorphing,
  decorations.pathreplacing,
  fit,
  matrix,
  patterns,
  shadows,
  shapes,
  shapes.geometric,
  shapes.gates.logic.IEC,
  shapes.gates.logic.US,
}
\tikzset{fontscale/.style = {font=\relsize{#1}}}
\newcommand{\alg}[1]{{\rmfamily \textsc{#1}}}
\usepackage{paralist}
\usepackage[textsize=footnotesize]{todonotes}
\usepackage{amsmath}
\usepackage{amsthm}
\usepackage{amssymb}
\usepackage{mathtools}
\DeclareMathOperator*{\argmax}{arg\,max}
\newtheorem{definition}{Definition}
\newtheorem{theorem}{Theorem}
\newtheorem{lemma}{Lemma}

\usepackage{hyperref}
\usepackage[capitalise]{cleveref}

\crefname{definition}{Def.}{Defs.}
\Crefname{definition}{Definition}{Definitions}

\title{Lifting DecPOMDPs for Nanoscale Systems --- A Work in Progress}
\author {
    Tanya Braun\textsuperscript{\rm 1},
    Stefan Fischer\textsuperscript{\rm 2},
    Florian Lau\textsuperscript{\rm 2},
    Ralf Möller\textsuperscript{\rm 1} \\
}
\affiliations {
    \textsuperscript{\rm 1} Institute of Information Systems, University of Lübeck, Lübeck, Germany \\
    \textsuperscript{\rm 2} Institute of Telematics, University of Lübeck, Lübeck, Germany \\
    \{tanya.braun, stefan.fischer, florian.lau, moeller\}@uni-luebeck.de
}

\begin{document}

\maketitle


\begin{abstract}
DNA-based nanonetworks have a wide range of promising use cases, especially in the field of medicine.
With a large set of agents, a partially observable stochastic environment, and noisy observations, such nanoscale systems can be modelled as a decentralised, partially observable, Markov decision process (DecPOMDP).
As the agent set is a dominating factor, this paper presents 
\begin{inparaenum}[(i)]
	\item lifted DecPOMDPs, partitioning the agent set into sets of indistinguishable agents, reducing the worst-case space required, and 
	\item a nanoscale medical system as an application.
\end{inparaenum}
Future work turns to solving and implementing lifted DecPOMDPs.
\end{abstract}


\noindent Particularly in times of medical crisis, precise and efficient diagnostic tools are invaluable.
The recent development of DNA-based nanonetworks shows promising results as a fast and robust diagnostics tool \cite{Lau21DNA} and a method for treating diseases in the human body \cite{LAU2019}.
A swarm of thousands of nanodevices collaborating to compute a diagnosis or treat a disease realise such tasks.
However, a network's environment is extremely heterogenous. 
Thus, the nanodevices have only local information about the system, which they are not able to consolidate into a global state, with their communication and computation capabilities highly limited and local information subject to noise. 
Under these conditions, a designer of a nanonetwork has to set up nanodevices functioning according to a certain plan or policy.
In addition, the designer needs to assess, e.g., the network's robustness regarding a diagnosis or the success rate of a joint action of medicine delivery.

Therefore, this paper formalises the setting as an offline decision making problem, specifically, a decentralised partially observable Markov decision process (DecPOMDP).
The formalisation fits well as a network contains a set of collaborating nanodevices, i.e.,\emph{agents}, with limited capabilities in a \emph{partially observable} environment whose transition can be approximated with a \emph{stochastic} process.
Computation time in DecPOMDPs depends exponentially on the number of agents in the worst case, which is problematic as nanonetworks can have hundreds of thousands of agents.

However, only a limited number of types of nanodevices exists, thus, partitioning the set of agents into subsets of agents whose behaviour is \emph{indistinguishable}.
The notion of \emph{lifting} refers to the idea of efficiently handling sets of indistinguishable objects using representatives \cite{Poo03}, reducing the theoretical dependency from exponential to polynomial w.r.t.\ these set sizes.
Here, we apply lifting to the agents of a DecPOMDP.
Specifically, the contributions of this paper are twofold:
\begin{inparaenum}[(i)]
	\item \emph{lifted DecPOMDPs} with a theoretical analysis in terms of space requirements and 
	\item a \emph{nanoscale medical system} as an application,
\end{inparaenum}
forming a first step towards our long-term goal of a full formalisation combined with a lifted solution approach that allows for quantifying a joint action's success rate.

Lifting has been successfully used for probabilistic inference \cite[see, e.g., ][]{BroTaMeDaRa11,AhmKeMlNa13,BraMo18b,HolMiBr19} as well as online decision making \cite{NatDo09,ApsBr11,GehBrMo19b,GehBrMo19d}, which uses decision and utility nodes in a lifted probabilistic graphical model (PGM).
Lifting can even bring tractability in terms of the set sizes of indistinguishable objects \cite{NieBr14}.  
In offline decision making, lifting has been used in calculations for relational descriptions of (PO)MDPs:
First-order MDPs \cite[FOMDPs,][]{BouRePr01,SanBo09} have a representation based on the situation calculus \cite{McC63}. 
\citeauthor{SanKe10} \shortcite{SanKe10} use lifting for pruning indistinguishable policies in their solution approach for a partially observable FOMDP.
Factorised FOMDP assume a factorisation of its representation, blurring the line towards online decision making with the factored representation \cite{SanBo07}.
First-order open-universe POMDPs follow the open-universe assumption for the first-order description of an environment using Bayesian logic as a basis \cite{SriRuRuCh14}.
Another first-order representation is the language of independent choice logic that also allows for a set of agents \cite{Poo97}.
To the best of our knowledge, we are the first to consider lifting the agent set, which finds its application in nanoscale systems where large groups of indistinguishable nanodevices act jointly in an environment.
We leave the environment encoded by a non-lifted joint distribution in this paper since the application at hand does not lend itself to a highly structured model.
Nonetheless, dealing with structure in the environment, relational or factorised, presents an exciting avenue for future work to extend its applicability.

The remainder of this paper is structured as follows:
We start with an introduction to nanoscale medical systems and recap offline decision-making from MDPs to DecPOMDPs.
Then, we present lifted DecPOMDPs, followed by a discussion of a lifted DecPOMDP modelling a nanosystem and next steps towards solving such a problem.
We end with a conclusion.



\section{Nanoscale Medical Systems}
DNA-based nanonetworks have been proposed as an alternative to polymerase chain reaction, PCR for short, for detecting arbitrary diseases on the basis of DNA.
In this scenario, a disease sample is mixed with a medical nanosystem that computes a programmed function depending on environmental parameters to decide if a disease is present.
This section explains the basic mechanisms and ideas behind this novel technology and sets up how DecPOMDPs can be used as a model for it. 

In general, a computational process at nanoscale is subject to resource constraints, and collaboration between nanodevices might be necessary to achieve a goal \cite{akyildiz2008nanonetworks}.
In \citeyear{SEEMAN1982237}, \citeauthor{SEEMAN1982237} first proposed DNA as a construction material for nanoscale systems.
Based on this idea, \citeauthor{Rothemund2006} \shortcite{Rothemund2006} developed the DNA-origami method, which allows for creating almost arbitrary shapes using DNA at nanoscale.
In \citeyear{andersen2009self}, \citeauthor{andersen2009self} created a box with a controllable lid using the DNA-origami method.
These boxes serve as a basis for medical nanosystem technology.
They can be filled with either medication or DNA-tiles that serve as an input for a computation \cite{winfree1998design}.
Certain DNA-tile systems form a Turing-complete computational model and are much more capable than most widely used medical diagnostic tools \cite{LAU2019}.
For a detailed introduction, we refer to \citeauthor{winfree1998design} \shortcite{winfree1998design}.  

\Cref{subfig:andtileset,subfig:assembly} show an example system of DNA-tiles that computes a 4-bit \alg{and} operation.
The blocks in \cref{subfig:andtileset} represent DNA-tiles.
They are modelled as non-rotatable blocks with colour-coded \emph{glues} at possibly all sides.
The number of black boxes represents the \emph{strength} of glue and a label next to it a condition.
Tiles with the same number of boxes and identical labels can form a binding.
That binding is subject to an environmental parameter called \emph{temperature} $\tau$.
The temperature $\tau$ encodes the necessary number of fitting glues over all neighbours to form a stable binding.
In the presented example, the temperature is $2$, and all tiles need at least two neighbours to stably bind together.

An assembly process begins with a \emph{seed-tile} $\sigma$, shown in \cref{subfig:assembly} at the right.
Due to the temperature requirement of $2$, three tiles can bind to the seed-tile $\sigma$: tile T, B, or 1.
The assembly process continues until the entire DNA molecule in \cref{subfig:assembly} is formed.
The \emph{receptors} R can only bind with the molecule if tiles 1, 2, 3, and 4 are present.
If those tiles are only conditionally present, they can represent the input values of a computation, in this case of a 4-bit \alg{and}.
The molecule can only fully assemble if all four tiles are present.
The other tiles are assumed to be present at all times in the medium.

\begin{figure}[t]
\begin{subfigure}{.98\columnwidth}
	\centering
	\begin{tikzpicture}
		\tiletype at (1.7,3.4,2022,$\sigma$,uzl_oceangreen_25,d,~,b,a)
		\tiletype at (0,3.4,0021,T,uzl_oceangreen_25,~,~,d,b)
		\tiletype at (1.7,5.1,2001,B,uzl_oceangreen_25,b,~,~,c)
		\tiletype at (0,0,0111,T,uzl_oceangreen_25,~,b,d,b)
		\tiletype at (0,1.7,1101,B,uzl_oceangreen_25,b,c,~,c)
		\tiletype at (1.7,0,1110,R,uzl_red_2,X,b,e,~)
		\tiletype at (1.7,1.7,1212,1,uzl_yellow_2,d,a,b,$a_1$)
		\tiletype at (3.4,3.4,1110,R,uzl_red_2,e,c,Y,~)
		\tiletype at (0,5.1,1212,2,uzl_yellow_2,d,$a_1$,b,$a_2$)
		\tiletype at (3.4,1.7,1212,3,uzl_yellow_2,d,$a_2$,b,$a_3$)
		\tiletype at (3.4,0,1210,4,uzl_yellow_2,e,$a_3$,e,~)
	\end{tikzpicture}
	\caption{Tileset for a 4 bit-\alg{And} message molecule.}
	\label{subfig:andtileset}
\end{subfigure}
\par\medskip
\begin{subfigure}{.98\columnwidth}
	\centering
	\begin{tikzpicture}
		\tiletype at (6,0,2022,$\sigma$,uzl_oceangreen_25,d,~,b,a)
		\tiletype at (6,1.2,0021,T,uzl_oceangreen_25,~,~,d,b)
		\tiletype at (6,-1.2,2001,B,uzl_oceangreen_25,b,~,~,c)
		\tiletype at (4.8,1.2,0111,T,uzl_oceangreen_25,~,b,d,b)
		\tiletype at (3.6,1.2,0111,T,uzl_oceangreen_25,~,b,d,b)
		\tiletype at (2.4,1.2,0111,T,uzl_oceangreen_25,~,b,d,b)
		\tiletype at (1.2,1.2,1110,R,uzl_red_2,X,b,e,~)
		\tiletype at (4.8,-1.2,1101,B,uzl_oceangreen_25,b,c,~,c)
		\tiletype at (3.6,-1.2,1101,B,uzl_oceangreen_25,b,c,~,c)
		\tiletype at (2.4,-1.2,1101,B,uzl_oceangreen_25,b,c,~,c)
		\tiletype at (1.2,-1.2,1110,R,uzl_red_2,e,c,Y,~)

		\tiletype at (4.8,0,1212,1,uzl_yellow_2,d,a,b,$a_1$)
		\tiletype at (3.6,0,1212,2,uzl_yellow_2,d,$a_1$,b,$a_2$)
		\tiletype at (2.4,0,1212,3,uzl_yellow_2,d,$a_2$,b,$a_3$)
		\tiletype at (1.2,0,1210,4,uzl_yellow_2,e,$a_3$,e,~)
	\end{tikzpicture}
	\caption{4 bit-\alg{And} message molecule.}
	\label{subfig:assembly}
\end{subfigure}
\caption{Tileset and finished product of an assembly process of DNA-tiles.}
\label{and}
\end{figure}
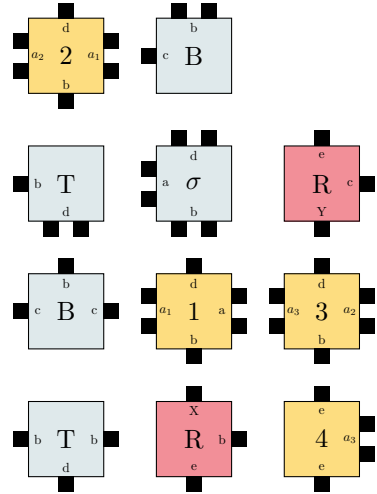
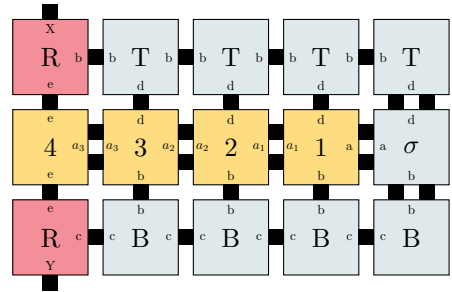

\begin{figure*}
\begin{center}
\begin{tikzpicture}[scale=1.0]

\path [every node/.style={
    anchor=west,
    text width=4cm,
}]
  node at (-1.25,9) {1. Marker detection}
  node at (-1.25,7.5) {2. Tile release}
  node at (-1.25,4) {3. Assembly}
  node at (-1.25,2) {4. Reception}
  node at (-1.25,0.5) {5. Medication}

  node at (8.75,9) {$\Bigg \}$ $k \geq 0$~markers}
  node at (8.75,7.5) {$\Bigg \}$ $i \gg 1$~nanosensors}
  node at (8.75,3.25) {$\Bigg \}$ $j \geq 0$~messages}
  node at (8.75,0.5) {$\Bigg \}$ $n \gg 1$~nanobots};

\path[every node/.style={
    shape=diamond,
    fill=uzl_orange_2,
    inner sep=1mm,
    rotate={rand*360},
}]
  (2.4, 9.2)  node (marker1) {}
  (3.4, 9.8)  node (marker2) {}
  (4.2, 9.7)  node (marker3) {}
  (5.1, 9.66) node (marker4) {}
  (6.5, 9.4)  node (marker5) {}
  (7.7, 9)    node (marker6) {};

\path[every node/.style={
  shape=star, star points=6, star point height=2.5mm,
  fill=uzl_oceanblue_1,
  draw=black!80,
  font=\smaller,
  inner sep=0mm,
  inner xsep=0mm,
  inner ysep=0mm,
  align=center,
}]
  (2.2, 7.7) node (sensor1) {Nano-\\sensor$_{1}$}
  (3.9, 8.2) node   (sensor2) {Nano-\\sensor$_{2}$}
  (5.6, 7.9) node (sensor3) {Nano-\\sensor$_{3}$}
  (7.4, 7.4) node (sensor4) {Nano-\\sensor$_{4}$};


\path[every node/.style={
  draw=black!80,
  font=\smaller,
  anchor=west,
  minimum width=3.5mm,
  minimum height=3.5mm,
  inner sep=0mm,
}]
  (2.6, 5.9) node [fill=uzl_yellow_2!50] (marker1a) {1}
  (3.3, 5.8) node [fill=uzl_yellow_2!50] (marker2a) {2}
  (4,   6.1) node [fill=uzl_yellow_2!50] (marker2b) {2}
  (4.6, 5.8) node [fill=uzl_yellow_2!50] (marker3a) {3}
  (5.4, 6.2) node [fill=uzl_yellow_2!50] (marker3b) {3}
  (6.1, 6.1) node [fill=uzl_yellow_2!50] (marker4a) {4}
  (6.8, 5.7) node [fill=uzl_yellow_2!50] (marker4b) {4}
  (1.3, 5.2) node [fill=uzl_oceangreen_10] (free1) {$\sigma$}
  (1.6, 6)   node [fill=uzl_red_2!50] (free6) {R}
  (2,   5)   node [fill=uzl_red_2!50] (free2) {R}
  (6.8, 4.1) node [fill=uzl_oceangreen_10] (free5) {T}
  (7.8, 4.3) node [fill=uzl_oceangreen_10] (free3) {T}
  (7.6, 5.3) node [fill=uzl_oceangreen_10] (free4) {$\sigma$}
  (3.3, 4.6) node [fill=uzl_oceangreen_10] (free7) {B};


\begin{scope}[nodes={
  draw=black!80,
  line width=0.1mm,
  font=\smaller,
  anchor=west,
  minimum width=3.5mm,
  minimum height=3.5mm,
  inner sep=0mm,
  column sep=-0.1mm,
  row sep=-0.1mm,
  }]

\foreach \x/\y/\n in {2.4/3.75/1, 6.9/2.7/2} {
  \path
  (\x, \y) node (msg\n) [matrix of nodes, ampersand replacement=\&,
  row 1 column 1/.style={nodes={fill=uzl_red_2!50}},
  row 1 column 2/.style={nodes={fill=uzl_oceangreen_10}},
  row 1 column 3/.style={nodes={fill=uzl_oceangreen_10}},
  row 1 column 4/.style={nodes={fill=uzl_oceangreen_10}},
  row 1 column 5/.style={nodes={fill=uzl_oceangreen_10}},
  row 2 column 1/.style={nodes={fill=uzl_yellow_2!50}},
  row 2 column 2/.style={nodes={fill=uzl_yellow_2!50}},
  row 2 column 3/.style={nodes={fill=uzl_yellow_2!50}},
  row 2 column 4/.style={nodes={fill=uzl_yellow_2!50}},
  row 2 column 5/.style={nodes={fill=uzl_oceangreen_10}},
  row 3 column 1/.style={nodes={fill=uzl_red_2!50}},
  row 3 column 2/.style={nodes={fill=uzl_oceangreen_10}},
  row 3 column 3/.style={nodes={fill=uzl_oceangreen_10}},
  row 3 column 4/.style={nodes={fill=uzl_oceangreen_10}},
  row 3 column 5/.style={nodes={fill=uzl_oceangreen_10}},] {
    R \& T \& T \& T \& T \\
    4 \& 3 \& 2 \& 1 \& $\sigma$ \\
    R \& B \& B \& B \& B \\
  };

  \path
      [fill=black!80]
      (msg\n) ++(-7.0mm,1.25mm)
      ++(0, 3.75mm)
      ++(-0.8mm, 0) rectangle +(1.4mm, 0.9mm)
      ++(0, -10mm) rectangle +(1.4mm, -0.9mm);
}


\path (4.5, 4.45) node (msg3) [matrix of nodes,draw=none,
  row 1 column 1/.style={nodes={fill=uzl_red_2!50}},
  row 1 column 2/.style={nodes={fill=uzl_oceangreen_10}},
  row 1 column 3/.style={nodes={fill=uzl_oceangreen_10}},
  row 1 column 4/.style={nodes={fill=uzl_oceangreen_10}},
  row 1 column 5/.style={nodes={fill=uzl_oceangreen_10}},
  row 2 column 1/.style={nodes={fill=uzl_yellow_2!50}},
  row 2 column 2/.style={nodes={fill=uzl_yellow_2!50}},
  row 2 column 3/.style={nodes={fill=uzl_yellow_2!50}},
  row 2 column 4/.style={nodes={fill=uzl_yellow_2!50}},
  row 2 column 5/.style={nodes={fill=uzl_oceangreen_10}},
  row 3 column 1/.style={nodes={fill=uzl_red_2!50}},
  row 3 column 2/.style={nodes={fill=uzl_oceangreen_10}},
  row 3 column 3/.style={nodes={fill=uzl_oceangreen_10}},
  row 3 column 4/.style={nodes={fill=uzl_oceangreen_10}},
  row 3 column 5/.style={nodes={fill=uzl_oceangreen_10}},] {
    &   &   &   & T \\
  4 & 3 & 2 & 1 & $\sigma$ \\
    &   & B & B & B \\
};

\end{scope}


\foreach \x/\y/\n in {3.5/2.1/1, 6/0.5/2} {

  \path [every node/.style={
      fill=uzl_oceanblue_1,
      draw=black!80,
      font=\smaller,
      align=center,
      anchor=east,
      minimum height=20mm,
      minimum width=12mm,
      inner sep=0mm,
        }]
    (\x, \y) node (machine\n) {Nano-\\bot};

  \path [every node/.style={
      draw=black!80,
      font=\smaller,
      align=center,
      minimum width=3.5mm,
      minimum height=3.5mm,
      inner sep=0mm,
  }]
    (\x, \y) ++(1.7mm, 7mm)  node (mRecU\n) {X}
    (\x, \y) ++(1.7mm, -7mm) node (mRecL\n) {Y};

  \path [fill=black!80, anchor=west]
    (\x, \y) ++(1.25mm, 4mm) ++(-0.7mm, 0) rectangle +(1.4mm, 0.7mm)
    (\x, \y) ++(1.25mm, -4mm) ++(-0.7mm, 0) rectangle +(1.4mm, -0.7mm);
}


\path [draw]
 (machine2) ++(-0.65,0.5) -- +(-0.65, 0.1)
            ++(0,   -1)  -- +(-0.65, -0.1);

\path [every node/.style={
    shape=isosceles triangle,
    minimum width=0.0mm,
    isosceles triangle apex angle=60,
    inner sep=0.6mm,
    fill=uzl_orange_2,
    rotate={rand*360},
  }]
   (machine2) ++(-0.75,0)
     node at +(-0.3, 0.2) (rel1) {}
     node at +(-0.6, -0.1) (rel2) {}
     node at +(-0.2, -0.3) (rel3) {};

\foreach \pos/\n in {machine1/5, machine2/4} {
\path [nodes={
  draw=black!80,
  line width=0.1mm,
  font=\smaller,
  anchor=west,
  minimum width=3.5mm,
  minimum height=3.5mm,
  inner sep=0mm,
  column sep=-0.1mm,
  row sep=-0.1mm,
  }]
  {(\pos.east) + (-0.01,0)} node (msg\n)
  [matrix of nodes, ampersand replacement=\&,
  row 1 column 1/.style={nodes={fill=uzl_red_2!50}},
  row 1 column 2/.style={nodes={fill=uzl_oceangreen_10}},
  row 1 column 3/.style={nodes={fill=uzl_oceangreen_10}},
  row 1 column 4/.style={nodes={fill=uzl_oceangreen_10}},
  row 1 column 5/.style={nodes={fill=uzl_oceangreen_10}},
  row 2 column 1/.style={nodes={fill=uzl_yellow_2!50}},
  row 2 column 2/.style={nodes={fill=uzl_yellow_2!50}},
  row 2 column 3/.style={nodes={fill=uzl_yellow_2!50}},
  row 2 column 4/.style={nodes={fill=uzl_yellow_2!50}},
  row 2 column 5/.style={nodes={fill=uzl_oceangreen_10}},
  row 3 column 1/.style={nodes={fill=uzl_red_2!50}},
  row 3 column 2/.style={nodes={fill=uzl_oceangreen_10}},
  row 3 column 3/.style={nodes={fill=uzl_oceangreen_10}},
  row 3 column 4/.style={nodes={fill=uzl_oceangreen_10}},
  row 3 column 5/.style={nodes={fill=uzl_oceangreen_10}},] {
    R \& T \& T \& T \& T \\
    4 \& 3 \& 2 \& 1 \& $\sigma$ \\
    R \& B \& B \& B \& B \\
  };
}


\path [draw=black!80,
  line width=0.15mm, densely dashed,
  >={Triangle[length=1.8mm, width=1.1mm]}, ->,
  shorten <=0.5mm,
  shorten >=1mm,
]
  (marker1) edge (sensor1)
  (marker2) edge (sensor2)
  (marker3) edge (sensor2)
  (marker4) edge (sensor3)
  (marker5) edge (sensor4)
  (marker6) edge (sensor4);

\path [draw=black!80,
  line width=0.15mm, densely dashed,
  shorten <=1mm,
  shorten >=0.5mm,
]
  (sensor1) edge (marker1a)
  (sensor2) edge (marker2a)
  (sensor2) edge (marker2b)
  (sensor3) edge (marker3a)
  (sensor3) edge (marker3b)
  (sensor4) edge (marker4a)
  (sensor4) edge (marker4b);

\path [draw=black!80,
  line width=0.15mm, densely dashed,
  >={Triangle[length=1.8mm, width=1.1mm]}, ->,
  shorten <=0.5mm,
  shorten >=3mm,
]
  (marker1a) edge (msg3)
  (marker2a) edge (msg3)
  (marker2b) edge (msg3)
  (marker3a) edge (msg3)
  (marker3b) edge (msg3)
  (marker4a) edge (msg3)
  (marker4b) edge (msg3);

\end{tikzpicture}

\caption[4 Bit-\alg{And}-nanonetwork]{Nanonetwork, which detects markers and assembles a 4 bit-\textsc{And} to achieve a distributed consensus.}
\label{scenario}
\end{center}
\end{figure*}
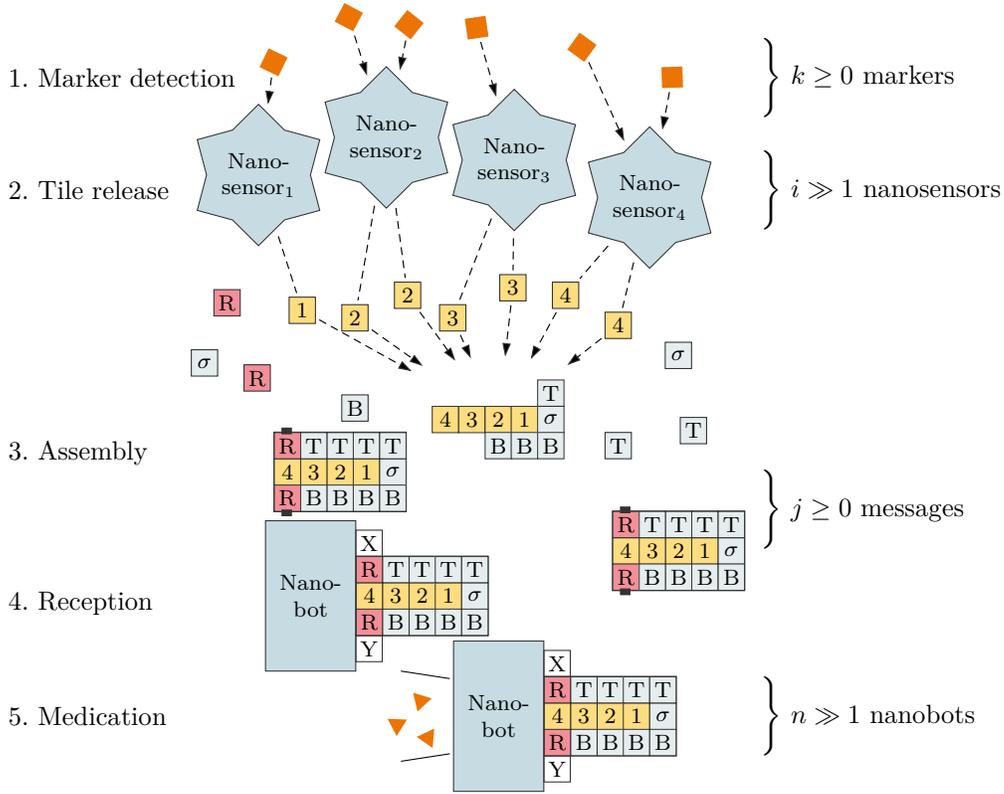

\Cref{scenario} shows the same DNA molecule assembly process incorporated into a network structure.
In the first phase, a possibly large number of predefined markers are detected by a very large number of indistinguishable nanosensors of four different types.
Upon detection, the lids of the nanosensors open, and they release their tiles 1--4 into the medium.
If all four tiles are present in high enough numbers, message molecules can fully assemble and later be detected by a number of indistinguishable nanobots.
Those react by predefined programming, e.g., releasing a fluorescence marker, medication, or additional tiles.
Many operations are possible depending on the desired response by the nanonetwork.

The entire assembly process is guided by diffusion and Brownian motion.
It can only be controlled by the types of supplied tiles, their concentration, and the environmental temperature.
Due to the stochastic nature, it is never fully clear where or when a message molecule forms and if it is erroneous. 
Further, the nanodevices themselves have no global and very limited local information that is also noisy about their environment.

Now that the basic functionality of DNA-based nanonetworks is clear, we can model part of them as DecPOMDPs.
The goal of this model is to coordinate the actions of the large number of indistinguishable devices in the nanonetwork to decide a predefined problem.
Since the nanobots are extremely resource-constrained, they only have partial information about the global state of the system.
In addition to that, communication in nanonetworks is expensive and should be reduced to a minimum.
One of the potential benefits of the proposed model is to get an estimate of the success rate of such a system.

\section{From Single Agent to Multiple Agents}
This section highlights the genesis of offline decision making from single-agent-fully-observable to multi-agent-partially-observable-joint-reward, which forms the basis for modelling nanoscale systems.
The underlying principle is that of \emph{maximum expected utility} (MEU) in which a utility function represents preferences over states.
We look through a PGM lens when setting up definitions, which are based on \citeauthor{RusNo20} \shortcite{RusNo20}, using random variables, $R$, which can take discrete values, referred to as range, $ran(R) = \{r_1, \dots, r_m\}$.
If $R$ is set to a value of its range $r \in ran(R)$, also called an \emph{event}, we denote it as $R=r$ or $r$ for short if $R$ is clear from its context.
Random variables that have actions as ranges are called \emph{decision random variables}.
For solving an MDP or one of its extensions, we focus on \emph{value iteration} as an example solution approach.
For a more detailed look into different solution techniques, please refer to \citeauthor{OliAm16} \shortcite{OliAm16} as a starting point with a focus on DecPOMDPs.

\paragraph{MDP}
An MDP is a sequential decision problem in a fully observable environment with a Markov-1 transition model and additive rewards, which implies that an agent's preferences are independent w.r.t.\ time.
\begin{definition}\label{def:mdp}
	An \emph{MDP} is a tuple $(S, A, T, R)$, with
	\begin{itemize}
		\item $S$ a random variable describing an agent's environment (the state space is the range of $S$), 
		\item $A_s$ a decision random variable with a set of actions for each state $s \in ran(S)$ as its range,
		\item $T(S',S,A_S) = P(S' \mid S, A_S)$ a \emph{transition model} (if $a \not\in ran(A_s)$ for a given $s$, then $P(s' \mid s, a) = 0$), and
		\item $R(S)$ a \emph{reward function}.
	\end{itemize}
	The solution to an MDP is a function $\pi: ran(S) \mapsto ran(A_S)$ called a \emph{policy}.
\end{definition}
The reward may also depend on the action applied as well as the resulting state, which does not change the problem itself in a major way.
Additionally, one can specify a \emph{discount} factor $\gamma \in ]0,1]$.
A discount factor close to $0$ signifies that future rewards are irrelevant.
The factor corresponds to an interest rate of ${1-\gamma}/{\gamma}$.
The utility of a state $s \in ran(S)$ can be described using the Bellman equation:
\begin{align}
	U(s) = R(s) + \gamma \max_{a \in ran(A_s)} \smashoperator[r]{\sum_{s' \in ran(S)}} P(s' \mid s, a) U(s'). \label{eq:bellman}
\end{align}
The inner sum adds up \emph{expected utilities} of each state $s'$, i.e., the utility of $s'$ multiplied with the probability of reaching $s'$ from $s$ applying an action $a$.
The optimal policy $\pi^*$ is then given by:
\begin{align*}
	\pi^*(s) = \argmax_{a \in ran(A_s)} \sum_{s' \in ran(S)} P(s' \mid s, a) U(s'),
\end{align*}
i.e., the agent chooses the $a$ that yields the maximum expected utility.
A standard approach to solving an MDP is called \emph{value iteration}.
It uses the bellman equation (\cref{eq:bellman}) in an update version:
\begin{align*}
	U_{t+1}(s) \gets R(s) + \gamma \max_{a \in ran(A_s)} \smashoperator[r]{\sum_{s' \in ran(S)}} P(s' \mid s, a) U_t(s'),
\end{align*}
updating the utilities until a stopping criterion is fulfilled, usually based on $\gamma$ and an allowed error $\epsilon$, e.g., $\max_s |U_{t+1}(s)-U_{t}(s)| < \epsilon {(1-\gamma)}/{\gamma}$.

\paragraph{POMDP}
If an agent cannot reliably observe or is not able to fully determine its state, the state is considered latent or partially observable, which leads to POMDPs.
\begin{definition}\label{def:pomdp}
	A \emph{POMDP} is a tuple $(S, A, T, R, O, \Omega)$ with 
	\begin{itemize}
		\item $S, A, T, R$ the components of an MDP, 
		\item $O$ a random variable with a set of possible observations as a range, and 
		\item $\Omega(O, S) = P(O \mid S)$ a \emph{sensor model}.
	\end{itemize}
	A \emph{belief MDP} is a tuple $(B, A, T, R, O, \Omega)$, i.e., a POMDP where the latent $S$ is replaced by random variable $B$ with an infinite set of belief states as range, with a \emph{belief state} $b \in ran(B)$ referring to a probability distribution over $S$.
	The state variables $S, S'$ in $T$, $R$, and $\Omega$ are replaced by belief state variables $B, B'$.
\end{definition}
The sensor model, analogously to the reward function, may also depend on an action as well as the outcome state.
To solve a POMDP, a corresponding belief MDP is solved, where optimal actions in an optimal policy depend on belief states (instead of actual states).
The solution to a belief MDP is called a conditional plan that combines \texttt{if\dots then \dots else} clauses, with the \texttt{if} conditions depending on observations and the bodies of the \texttt{if} and \texttt{else} parts containing actions.
Each conditional plan $p$ has a depth $d$ and is recursively constructed using plans of depth $d-1$:
\begin{align}
	U_p(s) = R(s) + \gamma \smashoperator{\sum_{s'\in ran(S)}} P(s' | s, a) \smashoperator{\sum_{o \in ran(O)}} P(o | s') U_{p.o}(s') \label{eq:condplan}
\end{align}
where $a$ is the initial action in $p$ and $p.o$ refers to the subplan of depth $d-1$ for percept $o$ that follows after $a$.
\Cref{eq:condplan} yields a value iteration algorithm.
All plans form hyperplanes in the belief space.
Those plans that have the highest values for some area of the space are dominating plans, whereas dominated plans have lower values over every possible belief state.
An integral part of value iteration is to eliminate dominated plans, which means for the recursive construction that fewer conditional plans need to be considered.

\paragraph{DecPOMDP}
In contrast to the single-agent setting so far, a DecPOMDP models a set of agents working jointly towards a common goal, a scenario relevant for nanoscale medical systems.
\begin{definition}\label{def:decpomdp}
	A \emph{DecPOMDP} is a tuple $(\boldsymbol{I},\allowbreak S,\allowbreak \{A_i\}_{i\in \boldsymbol{I}},\allowbreak T,\allowbreak R,\allowbreak \{O_i\}_{i\in \boldsymbol{I}},\allowbreak \Omega)$, with
	\begin{itemize}
		\item $\boldsymbol{I}$ a set of $N$ agents, 
		\item $S$ a random variable with a set of states as range, 
		\item $A_i$ a decision random variable with a set of local actions as range for each agent $i \in \boldsymbol{I}$, with $\boldsymbol{A} = \times_{i \in \boldsymbol{I}} ran(A_i)$ the set of \emph{joint actions},  
		\item $T(S', S, \boldsymbol{A}) = P(S' \mid S, \mathbf{A})$ a transition model,
		\item $R(S)$ a reward function, 
		\item $O_i$ a random variable with a set of local observations as range for each agent $i \in \boldsymbol{I}$, with $\boldsymbol{O} = \times_{i \in \boldsymbol{I}} ran(O_i)$ the set of \emph{joint observations}, and
		\item $\Omega(\boldsymbol{O}, S) = P( \boldsymbol{O} \mid S)$ a sensor model.
	\end{itemize}
	Each agent $i \in \boldsymbol{I}$ has a local policy $\pi_i$.
	The solution to a DecPOMDP is a \emph{joint policy} $\boldsymbol{\pi} = (\pi_i)_{i\in \boldsymbol{I}}$.
\end{definition}
In a DecPOMDP, each agent has its own set of actions and possible observations\footnote{$A_i = A_j$, $i,j \in \boldsymbol{I}$, is  possible but not mandatory. The same holds for the sets of local observations.} whereas the state and reward function are joint.
The joint state is usually assumed to not be fully observable, even if combining all local observations.
This is also the case for nanoscale systems.
If the joint state and reward function can be split up into mostly independent subspaces per agent, the DecPOMDP decomposes into a set of POMDPs that can be solved individually.
The joint reward function encodes that the set of agents receives a reward from the environment as a team in contrast to individual rewards for each agent. 
The generalisation of a DecPOMDP is a partially observable stochastic game in which each agent has its own reward function $R_i(S)$, which possibly conflicts with other agents' rewards.

A straight-forward way to find a solution to a DecPOMDP is to build conditional plans for each agent using value iteration and eliminate dominated plans considering all agents.
This approach quickly runs into memory problems because of the explosion of agent numbers, possible states, actions, and observations.
To formalise the space requirements in a DecPOMDP, we set up a simple lemma, which makes the combinatorial explosion by the agent numbers $N$ explicit.
\begin{lemma}\label{lem:decpomdp}
	The worst-case memory required for specifying a DecPOMDP of \cref{def:decpomdp} is exponential in the number of agents $N$.
\end{lemma}
\begin{proof}
	The transition model $T(S', S, \boldsymbol{A})$ and sensor model $\Omega(\boldsymbol{O}, S)$ of a DecPOMDP of \cref{def:decpomdp} have sizes $S_{T}^{dec}$ and $S_{\Omega}^{dec}$, respectively, that lie in 
	\begin{align}
		S_{T}^{dec} \in O( s \cdot s \cdot a^N) && \text{and} && S_{\Omega}^{dec} \in O (s \cdot o^N),	\label{eq:decpomdpsize}
	\end{align}
	with $s=|ran(S)|$, $a = \max_{i\in \{1, \dots, N\}|ran(A_i)|}$, and $o = \max_{i\in \{1, \dots, N\}|ran(O_i)|}$, which is exponential in the number of agents $N$. 
\end{proof}
For cases, in which the agent set carries structure in the form of partitions of indistinguishable agents, we present lifted DecPOMDPs as a step towards making the problem tractable.



\section{Lifting DecPOMDPs}
\label{sec:mdp}
Modeling a nanoscale system as a DecPOMDP yields a large set of agents $\boldsymbol{I}$ with partitions forming in which agents have the same setup regarding available actions and possible observations.
To get a handle on the large $\boldsymbol{I}$, we apply lifting to $\boldsymbol{I}$, using the identical setup of partitions for a compact encoding.
Before we present lifted DecPOMDPs, we need to consider under what conditions we can lift the agent set and what consequences this lifting has for the different components of a DecPOMDP.
Therefore, we first define a liftable DecPOMDP, then discuss its consequences, yielding a definition of a lifted DecPOMDPs.
We end with a look at worst-case space requirements.
\begin{definition}\label{def:liftabledecpomdp}
	A \emph{liftable DecPOMDP} is a DecPOMDP in which the set of agents $\boldsymbol{I}$ is partitioned into $K$ sets $\mathfrak{I}$, i.e., $\boldsymbol{I} = \bigcup_{k=1}^K\mathfrak{I}_k$ and $\forall k,l \in \{1, \dots, K\} : \mathfrak{I}_k \cap \mathfrak{I}_l = \emptyset$, where in each partition $\mathfrak{I}_k$, it holds that $\forall i,j \in \mathfrak{I}_k$:
	\begin{align}
		ran(A_i) = ran(A_j) \wedge ran(O_i) = ran(O_j). \label{eq:sameacts}
	\end{align}
\end{definition}
As a consequence of \cref{eq:sameacts}, it is sufficient to use $K$ variables $A_k$ and $O_k$ that apply to all agents in a partition $\mathfrak{I}_k$ instead of $A_i$ and $O_i$ for all $i \in \boldsymbol{I}$.
One may even explicitly model this pattern by using $A_k(X_k)$ and $O_k(X_k)$ with $X_k$ a logical variable that represents all agents in $\mathfrak{I}_k$ (for the use of logical variables in random variables, see, e.g., the work by \citeauthor{Tag13} \shortcite{Tag13}).

With a partitioning of the agent set, the joint observation $\boldsymbol{o}= (o_1, \dots, o_N)$ consists of $K$ \emph{partition observations} $\boldsymbol{o}_k = (o_{k,1}, \dots, o_{k,|\mathfrak{I}_{k}|})$.
Due to the ranges being discrete and bounded, each $\boldsymbol{o}_k$ can be lifted as each $o \in \boldsymbol{o}_k$ can only be in $ran(O_k)$.
Picking up the idea of histograms, which are used by the lifting tool of counting (see the work by \citeauthor{MilZeKeHaKa08}~\shortcite{MilZeKeHaKa08} for an introduction into counting), we can formally describe the setting as follows:
There are $r_k = |ran(O_k)|$ different possible observations per partition $\mathfrak{I}_k$, meaning, $\boldsymbol{o}_k$ can be encoded using a histogram:
\begin{align}
	\{(o_l, n_l)\}_{l=1}^{r_k}, o_l \in ran(O_k), n_l \in \mathbb{N}^0, \textstyle\sum_l n_l = |\mathfrak{I}_k|, \label{eq:hist}
\end{align}
or $[n_1, \dots, n_{r_k}]$ for short.
A histogram makes explicit a very basic insight: It does not matter which of the $|\mathfrak{I}_k|$ agents observe a particular value $o_l$, only how many observe $o_l$.
Basically, we have introduced a so-called counting random variable (CRV) for each partition with histograms as possible range values \cite{MilZeKeHaKa08}.
Briefly coming back to the representation of each partition using a logical variable $X_k$, the idea is that for a random variable parameterised with a logical variable such as $O_k(X_k)$, we count how often a particular range value is assigned to any grounding of $O_k(X_k)$, leading to a CRV $\#_{X_k}[O_k(X_k)]$, with histograms of \cref{eq:hist} as a range.
Given this insight into lifting partition observations, the joint observation over a partitioned agent set turns into
\begin{align}
	\boldsymbol{o} = (h_1, \dots, h_{K}), h_k \in ran(\#_{X_k}[O_k(X_k)]). 	\label{eq:seqhist}
\end{align}
The size of $\boldsymbol{o}$ reduces to $K \cdot \max_{k\in \{1,\dots,K\}}|ran(O_k)|$, which we assume to be much smaller than $N$.
If we can expect the observations per partition to consist of the same observation for whole partitions, then storing only one value per group would be enough (histograms would be \emph{peak-shaped} with one $n_l = |\mathfrak{I}_k|$ and all other $n_{l'} = 0$), with the size of $\boldsymbol{o}$ reducing to $K$.

Actions can be treated analogously to observations:
A joint action $\boldsymbol{a}$ is formalised as a sequence of partition actions $A_k$ encoded as histograms using a CRV $\#_{X_k}[A_k(X_k)]$. 
The size of $\boldsymbol{a}$ reduces from $N$ as well to $K \cdot \max_{k\in \{1,\dots,K\}}|ran(A_k)|$.
With actions, we also might be able to go a step further: 
For indistinguishable agents, the same action leads to a maximum expected utility (unless further constraints take effect), yielding that in a policy, the same action applies for a partition, which means peak-shaped histograms, and requiring only one action to store per partition.

Since we assume that the agents within a partition are indistinguishable, the consequence for the transition model $T$ and the sensor model $\Omega$ is that they both contain certain symmetric structures.
There are a number of joint observation $(o_1, \dots, o_N)$ that are encoded with the same sequence of histograms (\cref{eq:seqhist}), namely whenever, in each partition, the numbers $n_l$ in a histogram occur but with a different permutation of agents yielding them.
The same holds for joint actions.
As such, the transition model and the sensor model will map to the same probability for those cases.
Using a histogram allows for combining all the inputs mapping to the same probability into one input, reducing the size of $T$ and $\Omega$.
Before we formalize what reduction we can get, we combine the above lifting of actions, observations, and models into the following definition of a lifted DecPOMDP:
\begin{definition}\label{def:lifteddecpomdp}
	A \emph{lifted DecPOMDP} is a tuple $(\{\mathfrak{I}_k\}_{k},\allowbreak S,\allowbreak \{\#_{X_k}[A_k(X_k)]\}_{k},\allowbreak T,\allowbreak R,\allowbreak \{\#_{X_k}[O_k(X_k)]\}_{k},\allowbreak \Omega)$, $k=1$ to $K$, with
	\begin{itemize}
		\item $\{\mathfrak{I}_k\}_{k=1}^K$ a \emph{partitioning} of an agent set $\boldsymbol{I}$, $|\boldsymbol{I}| = N$, 
		\item $S$ a random variable with a set of states as range, 
		\item $\#_{X_k}[A_k(X_k)]$ a CRV with a set of histograms as in \cref{eq:hist} as range, with $\boldsymbol{A} = \times_{k = 1}^K ran(\#_{X_k}[A_k(X_k)])$ the set of joint actions,  
		\item $T(S', S, \boldsymbol{A}) = P(S' \mid S, \mathbf{A})$ a transition model,
		\item $R(S)$ a reward function, 
		\item $\#_{X_k}[O_k(X_k)]$ a CRV with a set of histograms as in \cref{eq:hist} as range, with $\boldsymbol{O} = \times_{k=1}^K ran(\#_{X_k}[O_k(X_k)])$ the set of joint observations, and
		\item $\Omega(\boldsymbol{O}, S) = P( \boldsymbol{O} \mid S)$ a sensor model.
	\end{itemize}
	Each partition $\mathfrak{I}_k$ has a local policy $\pi_k$.
	The solution to a DecPOMDP is a joint policy $\boldsymbol{\pi} = (\pi_k)_{k=1}^K$.
\end{definition}

\begin{theorem}\label{thm:size}
	The worst-case memory required for specifying a lifted DecPOMDP of \cref{def:lifteddecpomdp} is polynomial in the number of agents $N$.
\end{theorem}
\begin{proof}
	The transition model $T(S', S, \boldsymbol{A})$ and sensor model $\Omega(\boldsymbol{O}, S)$ of a lifted DecPOMDP of \cref{def:lifteddecpomdp} have sizes $S_{T}^{lif}$ and $S_{\Omega}^{lif}$, respectively, that lie in
	\begin{align*}
		S_{T}^{lif} \in O(s \cdot s \cdot a_l^K) && \text{and} && S_{\Omega}^{lif} \in O(s \cdot o_l^K),
	\end{align*}
	$s=|ran(S)|$, $a_l = \max_{k\in \{1, \dots,K\}} |ran(\#_{X_k}[A_k(X_k)])|$, and $o_l = \max_{k\in \{1, \dots,K\}} |ran(\#_{X_k}[O_k(X_k)])|$. 
	The size of the histogram space for a CRV $\#_{X}[R(X)]$ with a range size $r$ of $R$ and a domain size of $n$ of the logical variable $X$ is given by $\binom{n+r-1}{r-1}$, which is bounded by $n^r$ \cite{MilZeKeHaKa08}.
	Transferred to the lifted DecPOMDP setting, $a_l$ and $o_l$ can be described with $n^a$ and $n^o$, respectively, with $n = \max_{k \in \{1, \dots, K\}} |\mathfrak{I}_k|$, $a = \max_{k\in \{1, \dots, K\}|ran(A_k)|}$, and $o = \max_{k\in \{1, \dots, K\}|ran(O_k)|}$.
	Therefore, the worst case is given by 
	\begin{align}
		S_{T}^{lif} \in O(s \cdot s \cdot (n^a)^K) \text{ and } S_{\Omega}^{lif} \in O(s \cdot (n^o)^K), \label{eq:lifteddecpomdpsize}
	\end{align}
	which is no longer exponential compared to \cref{eq:decpomdpsize} in \cref{lem:decpomdp}, but polynomial in $n < N$.
\end{proof}
Assuming that $N \ll K$, then $S_T^{lif} \ll S_T^{dec}$ and $S_\Omega^{lif} \ll S_\Omega^{dec}$, even though $n^a > a$ and $n^o > o$.
If requiring only peak-shaped histograms, we have the best case: 
\begin{align*}
	S_{T}^{lif} \in O(s \cdot s \cdot a^K) && \text{and} && S_{\Omega}^{lif} \in O(s \cdot o^K),
\end{align*}
where the exponent of $N$ is replaced by the exponent of $K$ compared to \cref{eq:decpomdpsize}, which with $N \ll K$ really showcases the reduction in memory required.

Having the representation no longer depends exponentially on the number of agents facilitates an opening towards tractable inference in lifted DecPOMDPs w.r.t.\ the size of the agent set:
The lifted components may enable to also lift the calculations for a policy regarding these agents such that solving the inference problem of finding a policy no longer depends exponentially on the number of agents (tractability).
Straightforwardly, one could use the value iteration approach for DecPOMDPs by building conditional plans for each partition and prune plans over all partitions.

%


\section{Discussion}
\label{sec:dis}
This section presents a nanoscale medical system, as described earlier, modeled as a lifted DecPOMDP and discusses the next steps for the formalism.

\paragraph{A Nanoscale System Application}
To model a nanoscale medical system as a lifted DecPOMDP, we need to specify the components of a lifted DecPOMDP.
The following description models a nanoscale medical system as sketched in the brief introduction into such systems in the beginning of this paper.

The \emph{set of agents} $\boldsymbol{I}$ with its $K$ partitions consists of the different nanosensors and nanobots.
There are $\kappa$ types of nanosensors, each type reacting to one of $\kappa$ different markers.
So, for each type, there is a set of nanosensors, forming a partition $\mathfrak{I}_k^{\kappa}$ in $\boldsymbol{I}$.
For the nanobots, the setting is the same w.r.t.\ the types of messages that different types of nanobots react to.
With $\iota$ message types, there are $\iota$ sets of nanobots, each forming a partition $\mathfrak{I}_k^{\iota}$ in $\boldsymbol{I}$.
Preliminary experiments have shown that each partition may have around $64{,}000$ agents in such a nanoscale medical system, making the agent set at least of size $(\kappa + \iota) \cdot 64{,}000$.

Each type of agent basically has one action, which it can select to perform, and one possible observation.
So, in terms of the model, there are two actions in each partition: 
\begin{inparaenum}[(i)]
	\item outputting its load, which are tiles for nanosensors and medication for nanobots, or 
	\item doing nothing,
\end{inparaenum}
and two observations:
\begin{inparaenum}[(i)]
	\item for nanosensors, sensing a marker and for nanobots, receiving a message or
	\item sensing / receiving nothing.
\end{inparaenum}

The physical state space can be described in terms of the presence of markers and assembled messages of certain types.
Considering only assembled messages is a simplification as messages undergo a series of different states themselves during assembly, but only the assembled message is of importance for a nanobot.
With $\kappa$ different markers and $\iota$ different messages, there are $2^{\kappa} \cdot 2^{\iota}$ states.
Of course, other representations of the physical state space are possible, e.g., focusing on the medical context in a more detailed way.

With the given physical state space, the transition model $T$ would need to model the presence of markers, which could possibly follow a Poisson distribution, as well as the presence of messages, which we assume to follow a log-normal distribution.
To be able to compute a solution, further approximations might be necessary.
The overall goal is that nanobots output their medication if corresponding markers are present, which the reward function needs to encode.
The sensor model $\Omega$ would need to capture the probability of sensing correct inputs, which can vary greatly depending on outside influences.
In general, we have the following sources of error:
Nanosensors might sense a marker even though there is none or it might sense a marker of a wrong type as its own.
Nano\-agents might sense that they received a correct message even though the message is not there, the message is of another type, or the message is incorrectly assembled.
Both types of nanodevices might also mistake a received input as there not being an input.

To close out this application example, let us consider the worst case space requirements of a nanoscale medical system modelled as a lifted DecPOMDP:
If we consider four types of marker and one type of message, e.g., for detecting a specific disease, which means $\kappa = 4$ and $\iota = 1$, we have a state space of size $s = 2^4 \cdot 2^1 = 32$ whereas our agent set is of size $N = (4+1) \cdot 64{,}000 = 320{,}000$ partitioned into $K= 4 + 1 = 5$ partitions.
With these parameters and in reference to \cref{eq:decpomdpsize,eq:lifteddecpomdpsize}, the model sizes of $T$ are 
\begin{align}
	S_{T}^{dec} &\in O( 32 \cdot 32 \cdot 2^{320{,}000})  \label{eq:nanosize:t}\\
	S_{T}^{lif} &\in O( 32 \cdot 32 \cdot (64{,}000^2)^5) \label{eq:liftednanosize:t}\\
\intertext{and of $\Omega$ are}
	S_{\Omega}^{dec} &\in O (32 \cdot 2^{320{,}000}) \label{eq:nanosize:o}\\
	S_{\Omega}^{lif} &\in O(32 \cdot (64{,}000^2)^5) \label{eq:liftednanosize:o}
\end{align}
in the worst case.
\Cref{eq:nanosize:t,eq:liftednanosize:t,eq:nanosize:o,eq:liftednanosize:o} highlight to what a large degree the number of agents represent the dominating parameter in a lifted DecPOMDP with a nanoscale system as an application.


\paragraph{Work in Progress}
The next big step after the inception of lifted DecPOMDPs is specifying and implementing an approach to solving lifted DecPOMDPs.
A starting point lies in a value iteration approach for DecPOMDPs, which we lift for partitions of agents.
The hypothesis is that the goal of \emph{tractable inference w.r.t.\ agent numbers} is attainable given \cref{thm:size}.

In a next step, we focus on the environment representation $S$.
Assuming that there are a set of random variables $\{R_1, \dots, R_m\}$, with which we can describe the environment, the size of the state space is exponential in $m$, i.e., $O(r^m)$, with $r= \max_{i \in \{1, \dots, m\}}|ran(R_i)|$, which means for \cref{eq:decpomdpsize,eq:lifteddecpomdpsize} $s = r^m$.
There are two aspects to pursue, \emph{factorisation} and \emph{lifting}.
Factorisation refers to factorising a joint distribution into a set of local distributions exploiting (conditional) independences among the $n$ random variables for a compact encoding, allowing for a complexity of $O(n\cdot r^w)$ where $w$ refers to the so-called tree width, which basically denotes the largest number of arguments in an intermediate result during inference.
Lifting could then apply to that set of local distributions, further compactifying the encoding, which enables tractable inference in terms of domain sizes under certain conditions \cite{Tag13}.
For both aspects, the above mentioned (factorised) FO(PO)MDPs \cite{BouRePr01,SanBo07,SanKe10} as well as the work on lifted online decision \cite{GehBrMo19b,GehBrMo19d}, which uses lifted factorised models, are a jumping-off point.

Turning to the application, an inconspicuous assumption comes into focus, namely, that the set of agents $\boldsymbol{I}$ is known.
However, this assumption may not be true, possibly because the system designer cannot control how many agents arrive at a destination.
This fact is especially true in a nanoscale system where not only the exact number of agents is not known but also how many of those agents function correctly.
The usual practice of using an excessive number of agents that practically guarantees a minimum threshold of agents does not work with medical systems where too much medicine can be harmful.
Therefore, we need to keep partition sizes (and the overall number of agents) indefinite and possibly infer optimal sizes.
From a modelling standpoint, since the agent set is a discrete, bounded set, we can use a beta-binomial distribution with hyperparameters $\alpha$ and $\beta$ to model a probability distribution over possible set sizes for each partition.

The consequences for a given lifted DecPOMDP in terms of joint actions $\boldsymbol{A}$ and joint observations $\boldsymbol{O}$ lie in changed histograms.
The more interesting consequence arise in the transition model $T$ and sensor model $\Omega$ where the dimensions change.
The models currently do not have further structure, which makes this problem challenging.
Given a factorised representation with local distributions, inference in unknown universes \cite{BraMo19a} is a starting point.

\section{Conclusion}
\label{sec:concl}
This paper presents lifted DecPOMDPs, lifting the agent set, allowing for a reduction of the worst-case dependency of the model from exponential to polynomial in the number of agents.
Lifted DecPOMDPs work with a partitioning of the agent set where the agents of each partition are assumed to behave indistinguishably.
Therefore, we can use well-established lifting formalisms such as counting to reduce the length of joint actions and joint observations as well as the size of the transition and sensor models.
Lifted DecPOMDPs find their application in nanoscale systems, with the paper showcasing a medical diagnostics scenario.

Future work includes solving lifted DecPOMDPs, combining existing lifted solution approaches with the lifted representation, as well as recent advances in lifted online decision making.
From the nanosystem side, an important aspect lies in a more detailed model of a nanonetwork.
E.g., the stochastic behavior of the environment and the agents can be hard to model, making an analysis of network subclasses that can be compactly represented especially interesting.

\bibliography{bib/lifted_inf}

\end{document}